\documentclass[letterpaper, 10 pt, conference]{ieeeconf}

\IEEEoverridecommandlockouts                              

\usepackage[T1]{fontenc} 
\usepackage{amsmath}
\usepackage{amssymb}
\usepackage{dsfont}
\usepackage{graphicx,color}
\usepackage{cleveref,enumerate}
\usepackage{tikz}
\usetikzlibrary{shapes,arrows}

\newcommand{\R}{\mathbb{R}}
\newcommand{\G}{\mathcal{G}}
\newcommand{\E}{\mathcal{E}}
\newcommand{\V}{\mathcal{V}}
\newcommand{\N}{\mathcal{N}}
\renewcommand{\L}{\mathcal{L}}

\newcommand{\diag}{\text{diag}}
\renewcommand{\span}{\text{span}}
\newcommand{\gap}{\text{gap}}
\newcommand{\ones}{\mathbf{1}}
\newcommand{\zeros}{\mathbf{0}}

\usepackage{ntheorem}
\usepackage{cite}

\newtheorem{definition}{Definition}
\newtheorem{lemma}{Lemma}
\newtheorem{theorem}{Theorem}
\newtheorem{problem}{Problem}

\newtheorem{remark}{Remark}
\newtheorem{prop}{Proposition}

\newtheorem{coro}{Corollary}
\newtheorem{example}{Example}

\author{Mehran Zareh, Lorenzo Sabattini, and Cristian Secchi 
	\thanks{ 
		Authors are with the Department of Sciences and Methods for Engineering (DISMI), University of Modena and Reggio Emilia, Italy \tt\small{\{mehran.zareh, lorenzo.sabattini, cristian.secchi\}@unimore.it}}
}

\begin{document}
	\title{Decentralized Biconnectivity Conditions in Multi-robot Systems}
	\date{}
	\maketitle
	\begin{abstract}
		The network connectivity in a group of cooperative robots  can be easily broken if one of them loses its connectivity with the rest of the group. In case of having robustness with respect to one-robot-fail, the communication network is termed biconnected. In simple words, to have a biconnected network graph, we need to prove that there exists no articulation point. We propose a  decentralized approach that provides sufficient conditions for biconnectivity of the network, and we prove that these conditions are related to the third smallest eigenvalue of the Laplacian matrix. Data exchange among the robots is supposed to be neighbor-to-neighbor.
	\end{abstract}
	\section{Introduction}
		The last decade has witnessed a growing interest in decentralized control and decision making \cite{olfati2005consensus, Zareh_consensus,ZarehPhd15}. Recent developments have made it possible to have a large group of autonomous robots working cooperatively to perform complex tasks which are  simply not feasible by a single robot. Since the global information is not always available, control design for multi-robot systems based on local information exchange is a challenging task. Accordingly, the design of control systems has shifted from \emph{centralized} to \emph{decentralized}, where the information, locally collected by the units (robots), is processed in locus and control decisions are taken cooperatively by the robots with no supervision.

		Usually, the robots move in unknown environments with obstacles and they can get trapped. If the rest of the team continues moving far from a trapped robot, the communication between that robot and the team becomes weaker and finally the robot gets disconnected from the group. Therefore, sensing the connectivity and trying to preserve it, is a substantial task that must be seen as an objective of the control action. There are two main approaches to preserve the connectivity: the ones to maintain the local connectivity, and approaches to preserve the global connectivity.
		In local connectivity maintenance the aim is to develop a controller that keeps all initially existing communication links. Some examples of decentralized controller design for local connectivity maintenance algorithms can be found in \cite{notarstefano2006maintaining, ajorlou2010class}. Using these approaches, a proof for the network connectivity can be given. However, assuming the maintenance of every link is too restrictive, and several researchers have considered relaxations to the local connectivity maintenance such as assuming a spanning tree \cite{schuresko2009distributed}, and $k$-hop connectivity \cite{zavlanos2008distributed}. In comparison to the local ones, the global connectivity maintenance algorithms are based on global quantities of the network, and do not restrict link failures or creations (see e.g. \cite{zavlanos2011graph,sabattini2013decentralized,sabattini2013distributed,giordano2013passivity}).
		In \cite{de2006decentralized}, the authors propose a decentralized algorithm and quantify the
		connectivity property of the multi-agent system	with the second smallest eigenvalue of the state dependent Laplacian of the proximity graph. In \cite{franceschelli2013decentralized}, using an additional locally generated and communicated variable, a decentralized estimate  of the Laplacian spectrum is provided. In \cite{yang2010decentralized}, using a previously introduced decentralized estimator, the Fiedler vector and the algebraic connectivity are estimated.
		

		In order to achieve a robust communication in a team of cooperating mobile robots, the connectivity must be preserved when a single robot crashes or is suddenly called by a human user to perform  another task. This is equivalent to requiring that the network graph remains connected if one of the nodes and all its incident edges are removed. A graph with this property is said to be biconnected \cite{golumbic2004algorithmic}.  In addition to robustness, biconnectivity provides a better bandwidth for communication by providing multiple paths to the destination. The connectivity robustness of robot networks under failures is often neglected in the literature. Some related works in graph theory describe algorithms to find biconnected components in a graph based on optimization theories. The algorithms mainly utilize depth-first search or backtracking \cite{tarjan1972depth, tarjan1984finding} in a centralized way. In \cite{ahmadi2006distributed,ahmadi2006keeping}, the problem of biconnectivity check in a network is addressed. Although the algorithm is labeled distributed, the information exchange to make a table of connected and doubly-connected nodes is assumed, which imposes the nodes to exchange a big amount of information. In \cite{butterfield2008autonomous}, an algorithm to change a connected mobile robot graph into a biconnected configuration is proposed. Since the algorithm requires a global probe, it cannot be seen as a decentralized one. Very recently, \cite{ghedini2015improving} investigated the robustness problem in multi-robot systems so that, despite robot failures, most of them remain  connected and are able to continue the mission. Based on a maximum 2-hop communication, each robot is able to detect dangerous topological configurations in the sense of the connectivity and can mitigate in order to reach a new position to get a better connectivity level. The paper, based on local information, introduces a parameter, called vulnerability, that allows each robot to detect the level of its effect on the topological configuration.
		
		 
		In this paper, we provide algorithms to enable each node of the network graph to detect if it is a crucial one for the network connectivity, i.e., a node whose disconnection causes loss of connectivity of the graph. These nodes are termed as articulation points. To the best of the authors' knowledge, the problem of decentralized articulation point detection has not been studied by now. First, each robot perturbs its communication link weight. Then, based on matrix perturbation theory, the condition for each robot not to be an articulation point is achieved. Obviously, if there is no articulation point, the resulting graph is biconnected. We show that the graph biconnectivity is related to the third smallest eigenvalue of the Laplacian matrix.

		The rest of the paper is organized as follows. First, we introduce notations and some basic theorems and definitions on graph theory, which will be used in this work. Section \ref{section:problem_statement} introduces the main problem, and provides some essential definitions. Section \ref{section:main_results} provides the main contribution of this paper. We provide some theorems on perturbed communication weights to detect the articulation points with only 1-hop communications. In Section \ref{section:simulations} the simulation results are given to verify the theoretical findings. Finally, Section \ref{section:conclusions} concludes the paper, describes the open problems, and outlines the future directions.

		
	\section{Preliminaries}\label{section:Prob_for}

		In this section we recall some basic notions and definitions on graph theory and we introduce the notation used in the paper.
		
		The topology of bidirectional communication channels among the robots is represented by an undirected graph $\mathcal{G}=(\V, \E)$ where $\V=\{1,\ldots,n\}$ is the set of nodes (robots) and $\E\subset\V\times \V$ is the set of edges. An edge $(i, j) \in  \E$ exists if there is a
		communication channel between robots $i$ and $j$. Self loops
		$(i, i)$ are not considered. The set of robot $i$'s neighbors
		is denoted by $\N_i  = \{j \ : \ (j, i) \in  \E; j = 1, \ldots, n\}$. 
		The network graph $\mathcal{G}$ is encoded by the so-called {\em adjacency matrix}, an $n \times n$ matrix $A$ whose $(i,j)$-th entry $a_{ij}$ is greater than $0$ if $(i,j)\in \E$, $0$ otherwise. Obviously in an undirected graph matrix $A$ is symmetric.
		The degree matrix is defined as $D=\diag(d_1,d_2, \ldots, d_n)$ where	$d_i = \sum_{j=1}^{n}a_{ij}$ is the degree of node $i$. The Laplacian matrix of a graph is defined as $\L=D-A$. 
		The Laplacian matrix of a graph has several structural properties. It has non-negative real eigenvalues for any graph $\G$. Furthermore, let $\ones$ and $\zeros$ be respectively the vectors of ones and zeros with proper dimensions, then $\L\ones=\zeros$ and $\ones^T\L=\zeros^T$. Denote by $\lambda_i(\cdot)$ the $i$-th leftmost eigenvalue, and by $v_i(\cdot)$ and $w_i(\cdot)$ the right and left eigenvectors associated with $\lambda_i(\cdot)$. In this way, the eigenvalues of the Laplacian matrix can be ordered as
		$$0=\lambda_1(\L)\le\lambda_2(\L)\le \ldots\le \lambda_n(\L).$$
		 In $\G$ a node $i$ is reachable from a node $j$ if there exists an undirected path from $j$ to $i$. If $\G$ is connected then $\L$ is a symmetric positive semidefinite irreducible matrix.  Moreover, the algebraic multiplicity of the null eigenvalue of $\L$ is one.
		For a graph $\G$, the second smallest eigenvalue of the Laplacian matrix is called \emph{algebraic connectivity}. This eigenvalue gives a qualitative measure of connectedness of the graph. 
			Algebraic connectivity is a non-decreasing function of graphs with the same set of vertices. This means that, if $\G_1(\V, \E_1)$ and $\G_2(\V, \E_2)$ are two graphs constructed on the set $\V$  such that $\E_1 \subseteq \E_2$, then $\lambda_2(\L_1)\leq\lambda_2(\L_2)$. 
		In the other words,  the more connected the graph becomes the larger the algebraic connectivity will be. The corresponding eigenvector to the second smallest eigenvalue is called \emph{Fiedler vector}, which gives very useful information about the graph \cite{fiedler1975property}.  
		The next lemma explains a relation between the  eigenvectors of a Laplacian matrix.
		\begin{lemma}\cite{de2007old}\label{lemma:fiedler_prepend}
			Let $v_k(\L)$, $1<k\le n$, be a non-null eigenvector of the Laplacian matrix. Then
			\begin{equation}
				v_k^T(\L)\ones=0.
			\end{equation}
		\end{lemma}

We denote $\tilde{a}_i=[a_{ij}]^T\in \R^{n-1}, \ \ j=1,\ldots,n,j\ne i$. We also define the perturbed adjacency matrix $A^i(\epsilon)$ obtained from $A$ by multiplying all $a_{ij}$ and $a_{ji}$s by $\epsilon\in \R^+ $. The associated perturbed degree $D^i(\epsilon)=\diag(A^i(\epsilon)\ones)$ and Laplacian matrix $\L^i(\epsilon)=D^i(\epsilon)-A^i(\epsilon)$ are defined accordingly. We indicate the reduced graph $\G^{R_i}$ achieved from $\G$ by removing node $i$ and all its incident edges. Accordingly, $A^{R_i}$ is the adjacency matrix, $D^{R_i}$ is the degree matrix, and $\L^{R_i}$ is the Laplacian matrix of $\G^{R_i}$.	
	
Communications are assumed to be between each robot and its 1-hop neighbors. We assume that the network connectivity is guaranteed, and each robot can properly estimate the algebraic connectivity. For the connectivity maintenance conditions and algebraic connectivity estimation procedure, the readers are referred to \cite{yang2010decentralized,sabattini2013decentralized,franceschelli2013decentralized}.

\section{Problem Statement}\label{section:problem_statement}
Consider a network of $ n (> 2 ) $ robots whose interconnection
structure is modeled by an undirected graph $\G( \V , \E )$.

The following definitions from the algebraic graph theory  will be used in the rest of this paper.
\begin{definition}
	A vertex $i \in \V $ of a connected graph $\G$ is called an \emph{articulation point} if $\G^{R_i}$ is not connected. 
\end{definition}
\begin{definition}
	A connected graph is called \emph{biconnected} if it has no articulation point.
\end{definition}
\begin{definition}
	A \emph{block} in $\G$ is a maximal induced connected subgraph with no articulation point. If $\G$ itself is connected and has no articulation point, then $\G$ is a block \cite{west2001introduction}.
\end{definition}

\begin{definition}
	In a graph $\G(\V,\E)$, two paths between vertices $i, j\in \V$ are called \emph{internally disjoint} if they have no other vertices in common. 
\end{definition}
\begin{definition}
	In a graph $\G(\V,\E)$, two vertices $i, j\in \V$ are said to be \emph{doubly connected} $\iff$  there are two or more internally disjoint paths between them.
\end{definition}
The following lemma, from \cite{ahmadi2006keeping}, explains the relation between biconnectivity and doubly connected vertices.
\begin{lemma}\label{lemma:bicon_internaldisj}
	A given undirected graph $\G(\V,\E)$ is biconnected $ \iff $ any two vertices $i, j\in \V$ are doubly connected.
\end{lemma}

Now we are ready to define the main problem that we are going to study in this paper.

\begin{problem}
	For a multi-robot system with a connected interaction graph $\G$, find conditions based only on local data exchange so that there are more than one internally disjoint paths between any pair of nodes.  Equivalently, from Lemma \ref{lemma:bicon_internaldisj}, we are looking for the conditions under which the graph is biconnected.
\end{problem}

A very quick question that comes after the above problem is that if the graph is not biconnected, what strategies can bring the graph to the desired configuration.
We leave this problem for future studies.
\section{Main Contribution}\label{section:main_results}
To enable each single robot to be aware of its connectivity status in the graph, it needs to know the characteristics of the network graph when all its incident edges are disconnected. If the graph remains connected when the robot $i$ fails, then the node $i$ in the graph is not an articulation point. By putting weakly connected links between node $i$ and its neighbors, we aim at providing an estimate of the conditions after a complete disconnection. The proposed methodology includes the following steps
\begin{enumerate}[a)]
	\item First, we introduce an intermediate matrix $P^i(\epsilon)$, for each node $i$, whose eigenvalues are equal to the non-null ones of the perturbed Laplacian matrix $\L^i(\epsilon)$ with $\epsilon\in \R^+$ as a local design parameter  (Theorem \ref{theorem:P^R}).
	\item Then, we find an upper bound on the maximum gap between the pairs of the eigenvalues of this intermediate matrix and those of the reduced Laplacian matrix, $\L^{R_i}$ (Proposition \ref{prop:real_eig_F} and Lemma \ref{lemma:distanc_eig}).
	\item We provide some conditions on the third smallest eigenvalue of the perturbed Laplacian matrix so that the reduced graph $\G^{R_i}$ remains connected (Theorem \ref{theorem:algebraic_biconn}). 
	\item Finally, we demonstrate that, if the above conditions hold only for non-locally biconnected (defined later) nodes of $\G$, then $\G$ is biconnected (Proposition \ref{prop:locally_biconn} and Corollary \ref{coro:biconn}).  
\end{enumerate}

	\begin{theorem}\label{theorem:P^R}
		Given an undirected graph $\G(\V, \E)$ with $n$ nodes, for a given real scalar $\epsilon$,  the eigenvalues of the following matrix
	\begin{equation}\label{eq:L^R}
		 P^i(\epsilon)=\L^{R_i}+\epsilon \diag(\tilde{a}_i^T)+\epsilon\tilde{a}_i\ones ^T,
	\end{equation}
	are equal to non-null eigenvalues of $\L^i(\epsilon)$.
	\end{theorem}
	\begin{proof}
		Without loss of generality, assume that the node $i$ is the last node, that is, the associated elements in the adjacency, degree, and Laplacian matrices are in the last column and row. We can simply reformulate the perturbed adjacency matrix as
		\begin{equation}
		A^i(\epsilon)=\left[\begin{array}{cc}
		A^{R_i}&\epsilon\tilde{a}_i  \\ 
		\epsilon \tilde{a}_i^T&0
		\end{array} \right],
		\end{equation}
		and, subsequently,  the perturbed degree matrices as
		\begin{equation}
		D^i(\epsilon)=\diag(A^i(\epsilon)\ones)=\left[\begin{array}{cc}
		D^{R_i}+\epsilon \diag(\tilde{a}_i^T)& \zeros \\ 
		\zeros^T& \epsilon d_i 
		\end{array} \right].
		\end{equation}
		By simple calculations we get
		\begin{equation}
			\L^i(\epsilon)=D^i(\epsilon)-A^i(\epsilon)=\left[\begin{array}{cc}
			\L^{R_i}+\epsilon \diag(\tilde{a}_i^T)&-\epsilon\tilde{a}_i  \\ 
			-\epsilon \tilde{a}_i^T&\epsilon d_i 
			\end{array} \right].
		\end{equation}		
		Let $\lambda_i(\L^i(\epsilon))$ be a non-null eigenvalue of $\L^i(\epsilon)$ and  $v(\L^i(\epsilon))=\left[\begin{array}{c}
		v^1(\L^i(\epsilon))\\v^2(\L^i(\epsilon))	
		\end{array} \right]$, with $	v^1(\L^i(\epsilon))\in\R^{n-1}$, and $	v^2(\L^i(\epsilon))\in\R$, be a corresponding eigenvector. We have
		\begin{equation}
		\L^i(\epsilon) v(\L^i(\epsilon))=\lambda(\L^{i}(\epsilon)) v(\L^i(\epsilon)).
		\end{equation}
		or
		\begin{equation}\label{eq:LR_eigenvalue}
			\displaystyle\left\{\begin{array}{l}
			(\L^{R_i}	+\epsilon \diag(\tilde{a}_i^T))v^1(\L^i(\epsilon))-\epsilon\tilde{a}_iv^2(\L^i(\epsilon))= \\ \hspace{3cm}\vspace{3mm}
			 \lambda(\L^{i}(\epsilon)) v^1(\L^i(\epsilon))\\ 
			-\epsilon \tilde{a}_i^Tv^1(\L^i(\epsilon))+\epsilon d_iv^2(\L^i(\epsilon))=\\
			 \hspace{3cm}\vspace{3mm}\lambda(\L^{i}(\epsilon)) v^2(\L^i(\epsilon)).
			\end{array} \right.
		\end{equation}
		From Lemma \ref{lemma:fiedler_prepend} we can find a relationship between $v^1(\L^i(\epsilon))$  and $v^2(\L^i(\epsilon))$. We know that
		$$v^T(\L^i(\epsilon))\ones=0,$$
		hence
		\begin{equation}
			v^2(\L^i(\epsilon))=-v^{1^T}(\L^i(\epsilon))\ones =-\ones^T	v^1(\L^i(\epsilon)).
		\end{equation}
		By replacing in the first equation of \eqref{eq:LR_eigenvalue}, we obtain
		\begin{equation}
			(\L^{R_i}	+\epsilon \diag(\tilde{a}_i^T)+\epsilon\tilde{a}_i\ones^T)v^1(\L^i(\epsilon))= \lambda_i(\L^i(\epsilon)) v^1(\L^i(\epsilon)),
		\end{equation}
		which proves that $\lambda(\L^{i}(\epsilon))$ is an eigenvalue of the matrix $P^i(\epsilon)=(\L^{R_i}	+\epsilon \diag(\tilde{a}_i^T)+\epsilon\tilde{a}_i\ones )$, and the corresponding eigenvector is $v^1(\L^i(\epsilon))$.
	\end{proof}
	\begin{coro}\label{remark:1}
		If $\G$ is a connected graph, then $\L^i(\epsilon)$ has only one null eigenvalue and the other eigenvalues are positive. Then 
			$$\lambda_k(P^i(\epsilon))=\lambda_{{k+1}}(\L^i(\epsilon))\  \textit{for} \ k=1,\ldots,n-1.$$  
	\end{coro}

	Note $P^i(\epsilon)$ is achieved from $\L^{R_i}$ perturbed by the non-symmetric term $\epsilon(\diag(\tilde{a}_i)+\tilde{a}_i\ones^T)$. Before introducing a theorem to find an upper bound on the eigenvalue changes between these matrices, we need to show that any linear combination of them gets real eigenvalues. 
	\begin{prop}\label{prop:real_eig_F}
		For any given $\alpha, \beta\in \R$ and $\epsilon\in \R$ so that $\alpha^2+\beta^2\ne 0$,  the linear combination $F^i(\epsilon)=\alpha \L^{R_i} +\beta P^i(\epsilon)$ has real eigenvalues. 
	\end{prop}
	\begin{proof}
		See the Appendix.
	\end{proof}

	To ensure the connectivity of the network graph after a possible failure of  a node, we need to estimate the algebraic connectivity of the reduced graph. Obviously, if the second-smallest eigenvalue of the reduced Laplacian matrix is positive, then the reduced graph is connected. The next lemma introduces an important result from matrix perturbation theory, which enables us to find an upper bound on the distance between the pairs of the eigenvalues of two non-symmetrically perturbed matrices. 
	\begin{lemma}\label{lemma:distanc_eig}\cite{bhatia1987perturbation}
		Let $A$ be an $n\times n$  matrix with eigenvalues $\psi_1\ge\ldots\ge\psi_n$ and $B$ an $n\times n$ matrix with eigenvalues $\xi_1\ge\ldots\ldots\ge\xi_n$. Define the gap between the eigenvalues of these matrices as
		\begin{equation}\label{eq:gap1}
			\gap(A,B)=\max\limits_{j}|\psi_j-\xi_j|.
		\end{equation}
		If  all the real linear combinations of $A$ and $B$ have only real eigenvalues, then
		\begin{equation}\label{eq:gap2}
			\gap(A,B)\le\|A-B\|,
		\end{equation}
		where $\|\cdot\|$ indicates the Euclidean norm.
	\end{lemma}  

Now we are ready to introduce the main result of this paper. The next theorem provides some sufficient conditions for biconnectivity of a network based on finding a bound on the algebraic connectivity of reduced graphs.
  \begin{theorem}\label{theorem:algebraic_biconn}
		For a given multi-robot system with $n$ robots ($n> 2$), whose interaction network graph is modeled by an undirected connected graph $\G(\V,\E)$, the node $i$ is not an articulation point if, for a small $\epsilon\in \R^+$, we have
		\begin{equation}\label{eq:lambda_2_epsilon_bound}
		\lambda_{3}(\L^i(\epsilon)) >\epsilon \sqrt{n}\  (\sum\limits_{k=1}^{n}a^2_{ik})^{1/2}.
		\end{equation}
	\end{theorem}
	\begin{proof}
		  Notice that, if node $i$ is not an articulation point, then the reduced graph  $\G^{R_i}$ is a connected graph, and hence $\L^{R_i}$ has a positive second-smallest eigenvalue. Proposition \ref{prop:real_eig_F} shows that any linear combination of $\L^{R_i}$ and $P^i(\epsilon)$ has real eigenvalues. Therefore, due to Lemma \ref{lemma:distanc_eig}, the gap between the eigenvalues of $P^i(\epsilon)$ and $\L^R_i$ is bounded by 
		    \[\gap(P^i(\epsilon),\L{^R_i})\le\|P^i(\epsilon)-\L^{R_i}\|=\epsilon\|\diag(\tilde{a}_i)+\tilde{a}_i\ones^T\|.\]
		 It can be trivially shown that 
		 $$
		 \diag(\tilde{a}_i)+\tilde{a}_i\ones^T=\left[\begin{array}{ccc}
		 a_{11}&\ldots &a_{11}  \\ 
		 \vdots&  &\vdots  \\ 
		 a_{n-1,n-1}&\ldots  & a_{n-1,n-1}\end{array}
		 \right].  \\ 
		$$
		
		The Euclidean norm of a matrix is the square root of the sum of the squares of its elements. Hence $$\|\diag(\tilde{a}_i)+\tilde{a}_i\ones^T\|=\sqrt{n}\ (\sum\limits_{k=1}^{n}a^2_{ik})^{1/2}.$$
		 From \eqref{eq:gap1} and \eqref{eq:gap2}, we have
		 $$\|\lambda_{2}(P^i(\epsilon))-\lambda_2(L^{R_i})\|\le \epsilon\sqrt{n}\ (\sum\limits_{k=1}^{n}a^2_{ik})^{1/2}.$$
		 To prove the connectivity of $\G^{R_i}$ we need to show that $\lambda_2(\L^{R_i})>0$ or 
		 \[\lambda_{2}(P^i(\epsilon))>\epsilon\sqrt{n}(\sum\limits_{k=1}^{n}a^2_{ik})^{1/2}.\]
		 Since the graph is connected, from Remark \ref{remark:1} we get
		  \[\lambda_{2}(P^i(\epsilon))=\lambda_{3}(\L^i(\epsilon))>\epsilon\sqrt{n}(\sum\limits_{k=1}^{n}a^2_{ik})^{1/2}.\]
		 This means that, if we remove node $i$ from the network, it remains connected. In other words, if the condition \eqref{eq:lambda_2_epsilon_bound} is true, then node $i$ is not an articulation point. 
	\end{proof}

	Using Theorem \ref{theorem:algebraic_biconn} for all the nodes, if a decentralized eigenvalue estimation like the approach introduced in \cite{franceschelli2013decentralized} is implemented, then we only need local data to check the biconnectivity. However, in many multi-robot schemes, as formation control and rendezvous problems, the robots have some assigned tasks, and biconnectivity check introduces an extra effort to the robots that can lead to a loss in time and energy. On the other hand, when the biconnectivity check is necessary, an additional amount of energy or time loss is admitted. Therefore, if some of the robots can somehow sense that they are not in the risk of being an articulation point, they can skip the biconnectivity check. The next proposition can help each node to be aware of its own connectivity status to avoid unnecessary checks. 
	
	\begin{prop}\label{prop:locally_biconn}
		In a connected graph $\G(\V,\E)$, node $i$ is not an articulation point if the subgraph created on the set $\{i\}\cup \N_i$ forms a block. 
	\end{prop}  
	\begin{proof}
	Define $\V_1=\{i\}\cup \N_i$ and $\V_2=\V/\V_1$. Assume that the subgraph based on $\V_1$, namely $\G(\V_1)$ is a block. Due to the connectivity of the graph, there exists at least one path that connects each node in $\V_2$ to the block $\G(\V_1)$. Notice that there is no node in $\V_2$ adjacent to $i$ otherwise it would be in $\N_i$. Since the subgraph $\G(\V_1)$ is a block, we can conclude that the subgraph $\G(\V_1/{i})$ is connected. Consequently, the subgraph $\G((\V_1\cup\V_2=\V)/i)$ is connected. Therefore, from the definition, node $i$ is not an articulation point.      
	\end{proof}
 A node that satisfies Proposition \ref{prop:locally_biconn}, is called locally biconnected.
 \begin{remark}
 	In an undirected communication network graph, to characterize the local subgraph, each node only needs to receive the positions of its neighbors. Then, based on this model, the local  adjacency, degree, and Laplacian matrices can be determined. If the second smallest eigenvalue of the Laplacian matrix is positive then that node is locally biconnected.
 \end{remark}
 
 Now we can summarize our theorems by the following corollary. 
 \begin{coro}\label{coro:biconn}
 	A connected graph $\G(\V, \E)$ is biconnected if every node of $\G$ that is not locally biconnected meets the condition in  \eqref{eq:lambda_2_epsilon_bound}.
 \end{coro}

\section{Simulation results}\label{section:simulations}
In this section we present simulation results to verify the theoretical analysis. 
\begin{example}\label{example:1}
 We suppose that the communications are defined by the $R$-disk model, in which the elements of the adjacency matrix are defined as 
$$a_{ij}=\left\{\begin{array}{lc}
e^{-(\|p_i-p_j\|^2)/(2\sigma)}& \|p_i-p_j\|\le R  \\ 
0&  \|p_i-p_j\|> R,
\end{array} \right.$$ 
where $p_i$ indicated the position of robot $i$. For this simulation, we selected $R=0.5$ and $\sigma=0.125$. Consider the randomly generated network with $n=10$ in Figure \ref{fig:graphwith5robots}. We can see that the only node that is not locally biconnected (see Figure  \ref{fig:graphwith5robots_2}), is the one denoted by $*$. Hence, based on the proposed algorithm, this node starts doing a biconnectivity check. Selecting  $\epsilon=0.05$ gives $\lambda_3(\L^*(\epsilon))=0.034$, and $(\sum_{k=1}^{n}a^2_{*k})^{1/2}= 0.062$.  We can verify that the conditions in \eqref{eq:lambda_2_epsilon_bound} holds
$$\lambda_3(\L^*(\epsilon))=0.034>0.05\times\sqrt{10}\times 0.062=0.0098.$$
As we expected, the node $*$ meets the sufficient conditions to for not being an articulation point.
\end{example} 
Notice that the condition in \eqref{eq:lambda_2_epsilon_bound} is not necessary but sufficient. This means that, if we keep the same graph and we change the weights, \eqref{eq:lambda_2_epsilon_bound} might not hold anymore.
\begin{figure}
	\centering
	\includegraphics[scale=0.4]{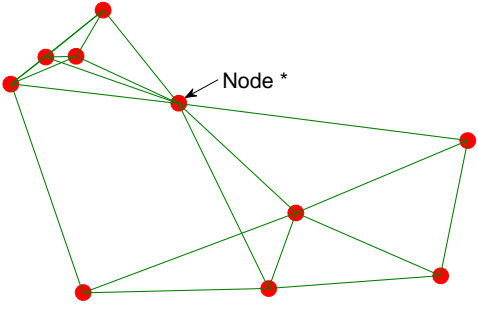}
	\caption{Communication graph in Example \ref{example:1}.}
	\label{fig:graphwith5robots}
\end{figure}

\begin{figure}
	\centering
	\includegraphics[scale=0.4]{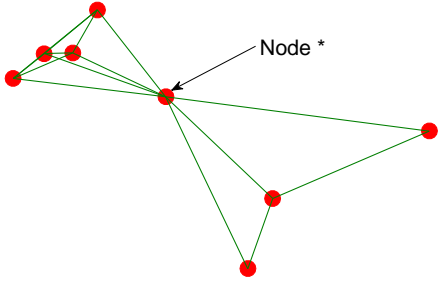}
	\caption{Local graph associated to node $*$ in Example \ref{example:1}}
	\label{fig:graphwith5robots_2}
\end{figure}

\section{Conclusions and Future works}\label{section:conclusions}
In this manuscript, a decentralized algorithm to determine the sufficient conditions for analyzing biconnectivity  was introduced. The definition of locally biconnected node was presented. We proved that, in order to have a  biconnected network, the nodes that are not locally biconnected must meet a special condition, one the third-smallest eigenvalue of the Laplacian matrix. This condition was obtained by making the nodes close to the disconnection. We also presented some theorems on the eigenvalues of non-symmetrically perturbed Laplacian matrix, and we used them to achieve the biconnectivity condition for the algorithm. In our future work, we are going to develop a decentralized protocol to obtain a biconnected network graph.
\appendices
\section*{Appendix}
\label{appendix:proof_real_eig_F}
In this section we provide the proof of Proposition \ref{prop:real_eig_F}. For this purpose, we need some preliminary manipulations and lemmas. Let  $\gamma=\alpha+\beta$ and $\eta=\beta\epsilon$. From \eqref{eq:L^R} we get
$$\begin{array}{rl}
F^i(\epsilon)=&(\alpha+\beta)\L^{R_i}+\beta\epsilon(\diag(\tilde{a}_i^T)+\tilde{a}_i\ones^T_{n-1})  \\ 
=& \gamma\L^{R_i}+\eta(\diag(\tilde{a}_i^T)+\tilde{a}\ones ^T).
\end{array} $$
For $\beta=0$, $F^i(\epsilon)$ becomes a symmetric matrix, and one can trivially show that it has real eigenvalues. So we need to prove the proposition for $\beta\ne 0$. Let $Q^i(\eta)=\gamma\L^{R_i}+\eta \tilde{a}_i\ones $. This gives
\begin{equation}\label{eq:F=Q}
	F^i(\epsilon= \eta/\beta)=Q^i(\eta)+\eta\diag(\tilde{a}_i^T).
\end{equation}
For convenience, hereafter we denote $F^i(\eta/\beta)$ by $F^i(\eta)$.

We recall the following lemma from the perturbation theory. 

\begin{lemma}\cite{seyranian2003multiparameter}\label{lemma:cai_eigenvalue_perturb}
	For a non-negative real number $\eta$, consider a matrix $M(\eta)$ and let $\lambda_1(M)=\ldots=\lambda_k(M), k\in \left[1, n\right]$ be a semi-simple eigenvalue\footnote{ An eigenvalue of a matrix
		is called semi-simple if its algebraic multiplicity is equal to its geometric multiplicity.} of $M(0)$. Denote by $v_1(M),\ldots, v_l(M)$ and $w_1(M),\ldots,w_l(M)$ associated right and left eigenvectors such that
	\[\left[\begin{array}{c}
	w_1^T(M)\\\ldots 
	\\ w_l^T(M)
	
	\end{array} \right]\left[\begin{array}{ccc}
	v_1(M)&\ldots  &v_l(M) 
	\end{array} \right]=I.\]
	Let $M^\prime =\frac{dM(\eta)}{d\eta}|_{\eta=0}$. Then the derivatives of the eigenvalues of $M$ with respect to $\eta$, $\displaystyle\frac{d\lambda(M)}{d\eta}|_{\eta=0}$, exist, and they are the eigenvalues of the following matrix
	\begin{equation}\label{eq:wTQv}
	\Delta=\left[\begin{array}{ccc}
	w_1(M)^TM^\prime v_1(M)&\ldots  &w_1(M)^TM^\prime v_l(M)  \\ 
	\vdots& \ddots & \vdots \\ 
	w_l(M)^TM^\prime v_1(M)& \ldots & w_l(M)^TM^\prime v_l(M)
	\end{array} \right]
	\end{equation}
\end{lemma} 
In order to prove Proposition \ref{prop:real_eig_F}, we introduce the following steps:

\begin{enumerate}[a)]
	\item In the first step, we characterize the eigenvalues of $Q^i(\eta)$, and we show they are all real.
	\item The second step is to demonstrate that the eigenvalues of $F^i(\eta)$ are all real. 
\end{enumerate}

\subsection*{Eigenvalues of $Q^i(\eta)$}
Note that $Q^i(\eta)$ is obtained by perturbing matrix $\gamma\L^{R_i}$ by $\eta\tilde{a}_i\ones ^T$. 

\begin{lemma}\label{theorem:Q}
	Let $\G$ be a connected graph and $\L^{R_i}$ be the Laplacian matrix of $\G^{R_i}$, with $l$ null eigenvalues $\lambda_1(\L^{R_i})=\lambda_2(\L^{R_i})=\ldots=\lambda_l(\L^{R_i})=0$. Then for the $k$-th eigenvalue of $Q^i$ we get
	
	\begin{equation}\label{eq:eig_Q_not0}
		\lambda_k(Q^i)=\gamma\lambda_k(\L^{R_i}), \ \ k=l+1, \ldots,n-1,
	\end{equation} 
	while for a small $\eta\in \R$ 
	
	\begin{equation}\label{eq:eig_Q_0}
		\lambda_k(Q^i)=0, \ \ k=2,\ldots,l,
	\end{equation}
	and $\lambda_1(Q^i)(\eta)$ gets a positive value. 
	\end{lemma}
	\begin{proof}
	For non-null eigenvalues of the reduced Laplacian matrix we have
		$$v_k^T(\L^{R_i})\ones=\ones^T v_k(\L^{R_i})=0,\ \ k=l+1,\ldots,n-1.$$
        Multiply $Q^i$ and $v_k(\L^{R_i})$		
		$$\begin{array}{c}
		Q^i(\eta)v_k(\L^{R_i})=\gamma\L_i^Rv_k(\L^{R_i})+\eta \tilde{a}_i\ones^T v_k(\L^{R_i}) \\=\gamma\lambda_k(\L^{R_i})v_k(\L^{R_i}), \ k=l+1,\ldots,n-1.
		
		\end{array} $$
		which shows that all the non-null eigenvalues of $Q^i(\eta)$ are equal to those of $\gamma\L^{R_i}$.
		
		The null eigenvalue of a Laplacian matrix is semi-simple \cite{bapat1998algebraic}. Let $\{v_1(\L^{R_i}),\ldots,v_l(\L^{R_i})\}$ be a set of orthogonal eigenvectors associated with the null eigenvalue of $\L^{R_i}$. Without loosing generality, let $v_1(\L^{R_i})=\ones$. Since $Q^i(0)=\gamma\L^{R_i}$ is symmetric, the left eigenvectors are equal to the right ones. Replacing $M^\prime=dQ_i(\eta)/d\eta=\tilde{a}_i\ones$, $w_1(M)=v_1(M)=\ones$, $w_k(M)=v_k(M)=v_k(\L^{R_i}), k=2,\ldots, l$ in \eqref{eq:wTQv}, knowing that $\ones^T v_j(\L^{R_i})=0, \ j=2,\ldots, l$, we get

		 \begin{equation}\label{eq:wTQv2}
		 \displaystyle
		  \begin{array}{l}
		  		  \displaystyle \Delta_{1,1}=w_1^TM^\prime v_1=\ones^T(\diag(\tilde{a}_i^T)\ones^T)\ones=(n-1)\sum_{k=1}^{n}a_{ik} \\ 
		 \displaystyle
		  \Delta_{m,j}=w_m^TM^\prime v_j=w_m^T(\diag(\tilde{a}_i^T)\ones^T)v_j=0\vspace{0.2cm} \\
		  \hspace{3cm}m,j=1,\ldots,l,\ \ (m,j)\ne(1,1)  \\ 		 
		  \end{array}.
		  \end{equation}
		Since the graph is supposed to be connected, then $\tilde{a}_{i}\ne \zeros\ \forall i$. Hence the matrix in \eqref{eq:wTQv} has only one non-null eigenvalue equal to $(n-1)\sum_{k=1}^{n}a_{ik}$. Consequently 
		\begin{equation}\label{eq:dlambda_deta}
			\frac{d\lambda_k(Q^i)}{d\eta}|_{\eta=0}=\left\{\begin{array}{cc}
			\displaystyle
			(n-1)\sum_{k=1}^{n}a_{ik}>0& k=1 \\ 
			\displaystyle
			0 & k=2, \ldots, l
			\end{array} \right.		
		\end{equation}
		
		This means that, by changing $\eta$ from zero, all the null eigenvalues of $Q^i(\eta)$ remain on the origin apart from one that moves to the right along the real axis. 
	\end{proof} 

So we proved that all the eigenvalues of $Q^i(\eta)$, and consequently the associated eigenvectors, get only real values.
\subsection*{Eigenvalues of $F^i(\eta)$}
 Now we want to show that the eigenvalues of $F^i(\eta)$ are all real and cannot get complex values.   
 \begin{proof}

Let $\lambda({F^i})$ be an eigenvalue of $F^i(\eta)$ and $v(F^i)$ be an associated eigenvector, that both can possibly get complex values. 
\begin{equation}\label{eq:F_eig}
	F^i(\eta)v({F^i})=\lambda(F^i)v({F^i}).
\end{equation}
For a small $\eta$ we can write 
\begin{equation}\label{eq:F_eig_sum}
\lambda({F^i})=\sum\limits_{k=0}^{\infty}\lambda_k({F^i})\eta^k,
\end{equation}
and
\begin{equation}\label{eq:F_eigvec_sum}
	v({F^i})=\sum\limits_{k=0}^{\infty}v_k({F^i})\eta^k.
\end{equation}
where $\lambda_k({F^i})$ and $v_k({F^i})$ may also take complex values.
We use induction to prove that all the eigenvalues of $F^i(\eta)$ are real. In this way, we first show that the statement is true for the first element ($\lambda_0(F^i)$ is real). Afterwards, we show that all the first $k-1$ elements are real, then the $k$-th eigenvalue is also real. 

From \eqref{eq:F=Q},  \eqref{eq:F_eig}, \eqref{eq:F_eig_sum}, and \eqref{eq:F_eigvec_sum}, we get
\begin{equation}
\begin{array}{l}
	(Q^i(\eta)+\eta \diag(\tilde{a}_i^T))\sum\limits_{k=0}^{\infty}v_k({F^i})\eta^k=  \\ 
\hspace{2cm}\sum\limits_{k=0}^{\infty}v_k({F^i})\eta^k\sum\limits_{k=0}^{\infty}\lambda_k({F^i})\eta^k.
\end{array} 
\end{equation}
To verify the equality for a non-zero $\eta$, the coefficients of all the exponents of $\eta$ must be equal in both the left and right side. For $k=0$ we get
$$Q^i(\eta)v_0(F^i)=\lambda_0({F^i})v_0(F^i),$$
This means that $\lambda_0({F^i})$ is an eigenvalue of $Q^i(\eta)$ with the associated left and right eigenvectors $v_0(F^i), w_0(F^i)$, and hence they are real.
For $k>0$, the equality of the two sides gives
 \begin{equation*}
 	Q^i(\eta)v_k({F^i})+\diag(\tilde{a}_i^T)v_{k-1}(F^i)=\sum\limits_{l=0}^{k}\lambda_l(F^i)v_{k-l}(F^i).
 \end{equation*}
By extracting the terms $0,1,$ and $k$ from the sum and doing some manipulations, we reach to 
\begin{equation}\label{eq:F_induction}
\begin{array}{l}
(Q^i(\eta)-\lambda_0(F^i)I)v_k({F^i})\\\hspace{1.5cm}+(\diag(\tilde{a}_i^T)-\lambda_1(F^i)I)v_{k-1}(F^i)  \\ 
\hspace{1.5cm}-\sum\limits_{l=2}^{k-1}\lambda_l(F^i)v_{k-l}(F^i)=\lambda_k({F^i})v_0(F^i). 
\end{array} 
\end{equation}
Now let $w_0(F^i)$ be a left eigenvector of $F^i(\eta)$ for $\lambda_0(F^i)$ so that $w_0^T(F^i)v_0(F^i)=1$. Then we have
$$w_0(F^i)^T(Q^i(\eta)-\lambda_0(F^i)I)=0.$$
By multiplying both sides of \eqref{eq:F_induction} by $w_0^T(F^i)$, the first term in the left side becomes zero, and we get
\begin{equation}\label{eq:F_induction2}
\begin{array}{l}
\displaystyle w_0^T(F^i)(\diag(\tilde{a}_i^T)-\lambda_1(F^i)I)v_{k-1}(F^i)
 \\\hspace{2cm}
-w_0(F^i)^T\sum\limits_{l=2}^{k-1}\lambda_l(F^i)v_{k-l}(F^i)\\\hspace{2cm}=\displaystyle\lambda_k({F^i})w_0^T(F^i)v_0(F^i)=\lambda_k({F^i}).
\end{array} 
\end{equation}
Notice that, if $\lambda_l(F^i)$ are real for $l=0,\ldots, k-1$ , then  $v_l(F^i)$ become all real valued. This implies that the left hand side of \eqref{eq:F_induction2} is a real number and consequently $\lambda_k({F^i})$ must get a real value. Therefore, we prove the proposition by induction.
 \end{proof}

\bibliographystyle{IEEEtran}
\bibliography {biblio_connectivity,biblio,biblio_applications}

\end{document}